\documentclass[twoside]{article}
\usepackage[accepted]{aistats2016}
\usepackage[round]{natbib}
\usepackage{hyperref}
\usepackage{url}

\usepackage{xcolor}
\hypersetup{
    colorlinks = true,
    linkcolor = cyan,
    linkbordercolor = {white},
    citecolor=cyan,
}

\usepackage{amsmath}
\usepackage{amsthm}
\usepackage{amsfonts}
\usepackage{amssymb}    
\usepackage{mathrsfs}
\usepackage{bbm, dsfont}
\usepackage{cleveref}
\usepackage{empheq}
\usepackage{graphicx} 
\usepackage{caption}
\usepackage{subcaption} 
\usepackage{floatrow} 
\usepackage{wrapfig}
\usepackage{multirow}
\usepackage{tikz}
\usepackage{tabu}
\usepackage[ruled,noline]{algorithm2e}

\theoremstyle{plain}
\newtheorem{thm}{Theorem}[section]

\newtheorem{proposition}[thm]{Proposition}

\usepackage[framemethod=TikZ]{mdframed}

\mdfdefinestyle{MyFrame}{%
    linecolor=white,
    outerlinewidth=1pt,
    roundcorner=2pt,
    innertopmargin=\baselineskip,
    innerbottommargin=\baselineskip,
    innerrightmargin=5pt,
    innerleftmargin=5pt,
    backgroundcolor=gray!20!white}

\theoremstyle{definition}

\newtheorem{example}{Example}[section]

\theoremstyle{remark}

\newcommand{\ie}[0]{\emph{i.e.},~}
\newcommand{\eg}[0]{\emph{e.g.},~}
\newcommand{\aka}[0]{a.k.a.~}

\newcommand{\wrt}{w.r.t.~}

\newcommand{\refSection}[1]{\Cref{#1}}
\newcommand{\refEq}[1]{\cref{#1}}

\newcommand{\boldit}[1]{\ensuremath{#1}}
\newcommand{\sfit}[1]{\ensuremath{#1}}
\renewcommand{\Re}[0]{\ensuremath{\mathbb{R}}}
\newcommand{\OO}[0]{\ensuremath{\mathcal{O}}}
\newcommand{\NNN}[0]{\ensuremath{\mathcal{N}}}
\newcommand{\xx}[0]{\ensuremath{\boldit{v}}}
\newcommand{\x}[0]{\ensuremath{{v}}}
\newcommand{\yy}[0]{\ensuremath{\boldit{h}}}
\newcommand{\y}[0]{\ensuremath{{h}}}
\newcommand{\Y}[0]{\ensuremath{{H}}}
\newcommand{\X}[0]{\ensuremath{{V}}}
\newcommand{\WW}[0]{\ensuremath{\boldit{W}}}
\newcommand{\W}[0]{\ensuremath{{W}}}
\newcommand{\pp}[0]{\ensuremath{\sfit{p}}}
\newcommand{\ff}[0]{\ensuremath{\sfit{f}}}
\newcommand{\ffinv}[0]{\ensuremath{\sfit{f}^{-1}}}
\newcommand{\FF}[0]{\ensuremath{\sfit{F}}}
\newcommand{\ZZ}[0]{\ensuremath{\sfit{A}}}
\newcommand{\DD}[0]{\ensuremath{\boldit{\sfit{D}}}}
\newcommand{\hh}[0]{\ensuremath{\sfit{r}}}
\newcommand{\gh}[0]{\ensuremath{\sfit{g}}}
\newcommand{\EE}[0]{\ensuremath{\mathbb{E}}}
\newcommand{\energy}[0]{\ensuremath{E}}
\newcommand{\YY}[0]{\ensuremath{\mathcal{H}}}
\newcommand{\dd}[0]{\ensuremath{\mathrm{d}}}
\newcommand{\nn}[1]{\ensuremath{^{(#1)}}}
\newcommand{\normal}[0]{\ensuremath{\mathcal{N}}}
\newcommand{\eps}[0]{\ensuremath{\varepsilon}}
\newcommand{\identt}[0]{\ensuremath{\mathbb{I}}}

\usepackage{authblk}
\title{Stochastic Neural Networks with Monotonic Activation Functions}

\author[1]{Siamak Ravanbakhsh, Barnab\'{a}s P\'{o}czos, Jeff Schneider}
\author[2]{\\Dale Schuurmans, Russell Greiner}
\affil[1]{Carnegie Mellon University, 5000 Forbes Ave, Pittsburgh, PA 15213}
\affil[2]{University of Alberta, Edmonton, AB T6G 2E8, Canada}
\date{}

\makeatletter
\def\blfootnote{\gdef\@thefnmark{}\@footnotetext}
\makeatother

\begin{document}

%

%

\twocolumn[

\maketitle


]

\begin{abstract}
We propose a Laplace approximation that creates a stochastic unit 
from any smooth monotonic activation function, using only Gaussian noise. 
This paper investigates the application of this stochastic approximation
in training a family of Restricted Boltzmann Machines (RBM) that are closely
linked to Bregman divergences.  
This family, that we call exponential family RBM (Exp-RBM),
is a subset of the exponential family Harmoniums that expresses
family members 
through a choice of smooth monotonic 
non-linearity for each neuron.
Using contrastive divergence along with our Gaussian approximation, we show that  
Exp-RBM can learn useful representations using novel stochastic units.
\end{abstract}

\section{Introduction}
\blfootnote{{\tiny Appearing in Proceedings of the 19th International Conference
on Artificial Intelligence and Statistics (AISTATS)
2016, Cadiz, Spain. JMLR: W\&CP volume 41. Copyright
2016 by the authors}}
Deep neural networks~\citep{lecun2015deep,bengio2009learning} have produced some of the best results in complex pattern recognition tasks where the training data is abundant. 
Here, we are interested in deep learning for generative modeling.
Recent years has witnessed a surge of interest in directed generative models that are 
trained using (stochastic) back-propagation~\citep[\eg][]{kingma2013auto,rezende2014stochastic,goodfellow2014generative}.
These models are distinct from deep energy-based models -- including deep Boltzmann machine~\citep{hinton2006fast} and (convolutional) deep belief network \citep{salakhutdinov2009deep,lee2009convolutional} -- that rely on a bipartite graphical model called restricted Boltzmann machine (RBM) in each layer. Although, due to their use of Gaussian noise, the stochastic units that we introduce in this paper can be potentially used with stochastic back-propagation, this paper is limited to applications in RBM.

To this day, the choice of stochastic units in RBM has been constrained to well-known members of the exponential family; in the past RBMs have used units with Bernoulli~\citep{smolensky1986information}, Gaussian~\citep{freund1994unsupervised,marks2001diffusion}, categorical~\citep{welling2004exponential}, Gamma~\citep{welling2002learning} and Poisson~\citep{gehler2006rate} conditional distributions.
The exception to this specialization, is the Rectified Linear Unit that was introduced with a (heuristic) sampling procedure~\citep{nair2010rectified}.

This limitation of RBM to well-known exponential family members is despite the fact that \citet{welling2004exponential} 
introduced a generalization of RBMs, called Exponential Family Harmoniums (EFH), 
covering a large subset of exponential family with bipartite structure. 
The architecture of EFH does not suggest a procedure connecting the EFH to \textit{arbitrary non-linearities} and more importantly a general sampling procedure is missing.\footnote{
As the concluding remarks of \cite{welling2004exponential} suggest, this capability is indeed desirable:``A future challenge is therefore to start the modelling process with the desired non-linearity and to subsequently introduce auxiliary variables to facilitate inference and learning.''}
We introduce a useful subset of the EFH, which we 
call exponential family RBMs (Exp-RBMs), 
with an approximate sampling procedure addressing these shortcomings.

The basic idea in Exp-RBM is simple: restrict the sufficient statistics to identity function. 
This allows definition of each unit using only its mean stochastic activation, which is the non-linearity of the neuron.
With this restriction, not only we gain interpretability, but also trainability; 
we show that it is possible to efficiently sample the activation of these stochastic neurons and train the resulting model using contrastive divergence.
Interestingly, this restriction also closely relates the generative training of Exp-RBM to  
 discriminative training using the matching loss and its regularization by noise injection.

In the following, \Cref{sec:model} introduces the Exp-RBM family and
\Cref{sec:learning} investigates learning of Exp-RBMs via an efficient approximate sampling procedure. 
Here, we also establish connections to discriminative training and
produce an interpretation of stochastic units in Exp-RBMs as an infinite collection of 
Bernoulli units with different activation biases. 
\Cref{sec:experiments} demonstrates the effectiveness of
the proposed sampling procedure, when combined with contrastive divergence training, in data representation.


\section{The Model}\label{sec:model}

The conventional RBM models the joint probability 
$\pp(\xx, \yy \mid \WW)$ for visible variables $\xx = [\x_1,\ldots,\x_i,\ldots,\x_I]$ with $\xx \in \mathcal{\X}_1 \times \ldots \times \mathcal{\X}_I $ and hidden variables $\yy = [\y_1,\ldots,\y_j,\ldots \y_J]$ with $\yy \in \mathcal{\Y}_1 \times \ldots\times \mathcal{\Y}_J$ as
\begin{align*}
\pp(\xx, \yy \mid \WW) = \exp(-\energy(\xx,\yy) - \ZZ(\WW)).
\end{align*}
This joint probability is a Boltzmann distribution with a particular energy function $\energy: \mathcal{\X} \times \mathcal{\Y} \to \Re$ and a normalization function $A$.
The distinguishing property of RBM compared to other Boltzmann distributions is the conditional independence due to its bipartite structure. 

\citet{welling2004exponential} construct Exponential Family Harmoniums (EFH), by first 
constructing independent distribution over individual variables: considering a hidden variable $\y_j$, its sufficient statistics $\{t_b\}_b$  and
canonical parameters $\{\tilde{\eta}_{j,b}\}_b$, this independent distribution is
\begin{align*}
  \pp(\y_j) = \hh(\y_j) \exp\bigg(\sum_{b} \tilde{\eta}_{j,b}\, t_b(\y_j) -\ZZ(\{\tilde{\eta}_{j,b}\}_b) \bigg)
\end{align*}
where $\hh: \mathcal{\Y}_j \to \Re$ is the \textit{base measure} and $\ZZ(\{\eta_{i,a}\}_a)$
is the normalization constant. 
Here, for notational convenience, we are assuming functions with distinct inputs are distinct -- \ie 
$t_b(\y_j)$ is not necessarily the same function as $t_b(\y_{j'})$, for $j' \neq j$.

The authors then combine these independent distributions using quadratic terms that reflect the bipartite structure of the EFH to get its joint form
\begin{align}
  \label{eq:EFH_joint}
  \pp(\xx, \yy) \propto &\exp\bigg(\sum_{i,a} \tilde{\nu}_{i,a} \,t_a(\x_i) \\
+ &\sum_{j,b} \tilde{\eta}_{j,b}\, t_b(\y_j) + \sum_{i,a,j,b} \W_{i,j}^{a,b} t_a(\x_i) t_b(\y_j) \bigg) \notag
\end{align}
where the normalization function is ignored and the base measures are represented as additional sufficient statistics with fixed parameters. In this model, the conditional distributions are 
\begin{align*}
  \pp(\x_i \mid \yy) = \exp\bigg( \sum_{a} {\nu}_{i,a} t_a(\x_j) -\ZZ(\{{\nu}_{i,a}\}_a \bigg)\\
  \pp(\y_j \mid \xx) = \exp\bigg( \sum_{b} {\eta}_{j,b} t_b(\y_j) -\ZZ(\{{\eta}_{j,b}\}_b \bigg)
\end{align*}
where the \textit{shifted} parameters ${\eta}_{j,b} = \tilde{\eta}_{j,b} + \sum_{i,a} \W_{i,j}^{a,b} t_a(\x_i)$ 
and ${\nu}_{i,a} = \tilde{\nu}_{i,a} + \sum_{j,b} \W_{i,j}^{a,b} t_b(\y_j)$
incorporate the effect of evidence in network on the random variable of interest. 

It is generally not possible to efficiently sample these conditionals (or the joint probability) for arbitrary sufficient statistics.
More importantly, the joint form of \refEq{eq:EFH_joint} and its energy function are ``obscure''. This is in the sense that   
the base measures $\{\hh\}$, depend on the choice of sufficient statistics and the normalization function $A(\WW)$. In fact for a fixed set of sufficient statistics $\{t_a(\x_i)\}_i, \{t_b(\y_j)\}_j$, different compatible choices of normalization constants and base measures may produce diverse subsets of the exponential family. Exp-RBM is one such family, where sufficient statistics are identity functions. 


\begin{table*}
  \caption{\small \it Stochastic units, their conditional distribution (\refEq{eq:bregman_p}) and the Gaussian approximation to this distribution. 
Here $\mathrm{Li}(\cdot)$
    is the polylogarithmic function and $\identt(\mathrm{cond.})$ is equal to one if the condition is satisfied and zero otherwise. $\mathrm{ent}(p)$ is the binary entropy function.}\label{table:units}
  \begin{center}
    \scalebox{.7}{
      \begin{tabu}{|[2pt] c |[2pt]  c | c | c  |[2pt]}\hline
        unit name & non-linearity  $\ff(\eta)$    & Gaussian approximation & conditional dist $\pp(\y \mid \eta)$  \\\tabucline[2pt]{-}
        Sigmoid (Bernoulli) Unit & ${(1 + e^{-\eta})^{-1}}$&      -                 & $\exp \{ \eta \y -  \log(1 + \exp(\eta))\}$   \\\hline
        Noisy Tanh Unit   & ${(1 + e^{-\eta})^{-1}} - \frac{1}{2}$& $\NNN(\ff(\eta) , (\ff(\eta) - 1/2)(\ff(\eta) + 1/2))$ & $\exp \{ \eta \y - \log(1 + \exp(\eta)) + \mathrm{ent}(\y) - \gh(\y) \}$\\\hline
        ArcSinh Unit & $\log(\eta + \sqrt{1 + \eta^2})$ & $\NNN(\mathrm{sinh}^{-1}(\eta) , (\sqrt{1 + \eta^2})^{-1})$ & $\exp \{ \eta \y - \mathrm{cosh}(\y) + \sqrt{1 + \eta^2} - \eta\, \mathrm{sin}^{-1}(\eta) - \gh(\y) \}$ \\\hline
        Symmetric Sqrt Unit (SymSqU)  & $\mathrm{sign}(\eta)\sqrt{|\eta|}$ & $\NNN(\ff(\eta) ,  \sqrt{|\eta|}/2)$ & $\exp \{ \eta \y - |\y|^3/3 - 2(\eta^{2})^{\frac{3}{4}}/3 - \gh(\y) \}$ \\\hline
        Linear (Gaussian) Unit & $\eta$                 & $\NNN(\eta , 1)$       & $\exp \{ \eta \y - \frac{1}{2}(\eta^2) - \frac{1}{2}(\y^2) -\log(\sqrt{2\pi}) \}$\\\hline
        Softplus Unit & $\log(1 + e^{\eta})$ & $\NNN(\ff(\eta) , (1 + e^{-\eta})^{-1})$ & $\exp \{ \eta \y + \mathrm{Li}_{2}(-e^{\eta}) + \mathrm{Li}_2(e^{\y}) + \y \log(1 - e^{\y}) - \y \log(e^{\y} - 1) - \gh(\y)\}$ \\\hline
        Rectified Linear Unit (ReLU)  & $\max(0,\eta)$ & $\NNN(\ff(\eta) , \identt(\eta > 0))$ &  - \\\hline
        Rectified Quadratic Unit (ReQU)  & $\max(0,\eta |\eta|)$ & $\NNN(\ff(\eta) , \identt(\eta > 0) \eta)$ & - \\\hline
        Symmetric Quadratic Unit (SymQU)  & $\eta |\eta|$ & $\NNN(\eta |\eta| ,  |\eta|)$ & $\exp \{ \eta \y - |\eta|^3/3 - 2(\y^{2})^{\frac{3}{4}}/3 - \gh(\y) \}$\\\hline
        Exponential Unit  & $e^{\eta}$ & $\NNN(e^\eta , e^\eta)$ & $\exp \{ \eta \y - e^{\eta} - \y (\log(y) - 1) - \gh(\y) \}$\\\hline
        Sinh Unit & $\frac{1}{2}(e^{\eta} - e^{-\eta})$ & $\NNN(\mathrm{sinh}(\eta) , \mathrm{cosh}(\eta))$ & $\exp \{ \eta \y - \mathrm{cosh}(\eta) + \sqrt{1 + \y^2} - \y\, \mathrm{sin}^{-1}(\y) - \gh(\y) \}$\\\hline
        Poisson Unit  & $e^{\eta}$ & - & $\exp \{ \eta \y - e^{\eta} - \y! \}$ \\\hline
      \end{tabu}
    }
  \end{center}
\end{table*}

\subsection{Bregman Divergences and Exp-RBM}\label{sec:bregman}
Exp-RBM restricts the sufficient statistics $t_a(\x_i)$ and $t_b(\y_j)$ to single
identity functions $\x_i, \y_j$ for all $i$ and $j$. This means the RBM has a single weight matrix $\WW \in \Re^{I \times J}$. As before, each hidden unit $j$, receives an input $\eta_j = \sum_{i} W_{i,j} \x_i$ and similarly each visible unit $i$ receives the input $\nu_i = \sum_{j} W_{i,j} \y_j$.
\footnote{Note that we ignore the ``bias parameters'' $\tilde{\nu}_i$ and $\tilde{\eta}_j$, since they can be encoded using the weights for additional hidden or visible units ($\y_j = 1, \x_i = 1$)  that are clamped to one.}

Here, the conditional distributions $\pp(\x_i \mid \nu_i)$ and $\pp(\y_j \mid \eta_j)$ 
 have a single \textit{mean parameter}, $\ff(\eta) \in \mathcal{M}$, which is equal to the mean of the conditional distribution. We could freely assign any desired continuous and monotonic non-linearity $\ff: \Re \to \mathcal{M} \subseteq \Re$ to represent the mapping from canonical parameter $\eta_j$ to this mean parameter: $\ff(\eta_j) = \int_{\mathcal{\Y}_j} \y_j \pp(\y_j \mid \eta_j)\, \mathrm{d} \y_j$.
This choice of $\ff$ defines the conditionals
\begin{align}
  \label{eq:bregman_p}
  \pp(\y_j \mid \eta_j) &= \exp \bigg( -\DD_{\ff}(\eta_j\,\|\, \y_j) - \gh(\y_j) \bigg) \\
  \pp(\x_i \mid \nu_i) &= \exp \bigg( -\DD_{\ff}(\nu_i \,\|\, \x_i) - \gh(\x_i) \bigg)\notag
\end{align}
where $\gh$ is the base measure and $\DD_{\ff}$ is the Bregman divergence for the function $\ff$.

The Bregman divergence~\citep{bregman1967relaxation,banerjee2005clustering} between $\y_j$
and $\eta_j$ for a monotonically increasing transfer function (corresponding to the activation function) $\ff$ is given by\footnote{
The conventional form of Bregman divergence is
   $\DD_{\ff}(\eta_j\, \|\,\y_j ) = \FF(\eta_j) - \FF(\ffinv(\y_j)) - \y_j (\eta_j - \ffinv(\y_j))$,
where $\FF$ is the anti-derivative of $\ff$.
Since $\FF$ is strictly convex and differentiable, it has a Legendre-Fenchel dual
$\FF^*(\y_j) = \sup_{\eta_j} \langle \y_j, \eta_j\rangle - \FF(\eta_j)$.
 Now, set the derivative of the r.h.s.  \wrt $\eta_j$ to zero to get $\y_j = \ff(\eta_j)$, or $\eta_j = \ffinv(\y_j)$, where $\FF^*(\y_j)$ is the anti-derivative of $\ffinv(\y_i)$. 
Using the duality to switch $\ff$ and $\ffinv$ in the above we can get $\FF(\ffinv(\y_j)) = \y_j \ffinv(\y_j) - \FF^*(\y_j)$. By replacing this in the original form of Bregman divergence we get the alternative form of \Cref{eq:bregman}.}
\begin{align}
  \label{eq:bregman}
  \DD_{\ff}(\eta_j\, \|\, \y_j) = - \eta_j \y_j + \FF(\eta_j) + \FF^*(\y_j)
\end{align}
where $\FF$ with $\frac{\dd}{\dd \eta} \FF(\eta_j) =  \ff(\eta_j)$ is the anti-derivative of $\ff$
and $\FF^*$ is the anti-derivative of $\ffinv$. Substituting this expression for Bregmann divergence in
\refEq{eq:bregman_p}, we notice both $\FF^*$ and $\gh$ are functions of $\y_j$. 
In fact, these two functions are often not separated \citep[\eg][]{mccullagh1989generalized}. By separating them
we see that some times, $\gh$ simplifies to a constant, enabling us to approximate \Cref{eq:bregman_p} in \Cref{sec:sampling}.

\begin{mdframed}[style=MyFrame]
\begin{example}\label{ex:gaussian}
  Let $\ff(\eta_j) = \eta_j$ be a linear neuron. Then $\FF(\eta_j) = \frac{1}{2} \eta_j^2$ and 
  $\FF^*(\y_j) = \frac{1}{2} \y_j^2$, giving a Gaussian conditional distribution
  $\pp(\y_j \mid \eta_j) = e^{- \frac{1}{2}( \y_j - \eta_j)^2  + \gh(\y_j) } $, where 
  $\gh(\y_j) = \log(\sqrt{2\pi})$ is a constant. 
\end{example}
\end{mdframed}

\subsection{The Joint Form}
So far we have defined the conditional distribution of our Exp-RBM as members of, 
using a single mean parameter $\ff(\eta_j)$ (or $\ff(\nu_i)$ for visible units) that represents the activation function of the neuron. Now we would like to find the corresponding joint form and the energy function.

The problem of relating the local conditionals to the joint form in graphical models goes back to the work of \citet{besag1974spatial}.
It is easy to check that, using the more general treatment of \citet{yang2012graphical}, 
the joint form corresponding to the conditional of \refEq{eq:bregman_p} is 
\begin{align}
  \label{eq:joint}
  &\pp(\xx, \yy \mid \WW) \, = \, \exp \bigg ( \xx^T \cdot \WW \cdot \yy  \\
 &- \sum_{i} \big(\FF^*(\x_i) + \gh(\x_i)\big) 
 - \sum_{j} \big(\FF^*(\y_j) + \gh(\y_j) \big) - \ZZ(\WW) \bigg ) \notag
\end{align}
where $\ZZ(\WW)$ is the joint normalization constant. It is noteworthy that only the anti-derivative of $\ffinv$, $\FF^*$ appears in the joint form and $\FF$ is absent.  
From this, the energy function is
\begin{align}\label{eq:energy}
   &\energy(\xx, \yy) = 
  - \xx^T \cdot \WW \cdot \yy  \\
 &+ \sum_{i} \big(\FF^*(\x_i) + \gh(\x_i)\big) + \sum_{j} \big(\FF^*(\y_j) + \gh(\y_j) \big).  \notag
 \end{align}
\begin{mdframed}[style=MyFrame]
\begin{example}
  For the sigmoid non-linearity $\ff(\eta_j) = \frac{1}{1 + e^{-\eta_j}}$, 
  we have $\FF(\eta_j) = \log(1 + e^{\eta_j})$ and $\FF^*(\y_j) = (1 - \y_j)\log(1 - \y_j) + \y_j \log(\y_j)$ is the negative entropy.
  Since $\y_j \in \{0, 1\}$ only takes extreme values, the negative entropy  $\FF^*(\y_j)$ evaluates to zero: 
\begin{equation}\label{eq:sigmoid_p}
\pp(\y_j \mid \eta_j) = \exp \bigg ( \y_j \eta_j - \log(1 + \exp(\eta_j)) - \gh(\y_j)\bigg)
\end{equation}
  Separately evaluating this expression for $\y_j = 0$ and $\y_j = 1$, shows that the above conditional is a well-defined distribution for $\gh(\y_j) = 0$, and in 
fact it turns out to be the sigmoid function itself -- \ie 
  $\pp(\y_j = 1 \mid \eta_j) = \frac{1}{1 + e^{-\eta_j}}$. 
  When all conditionals in the RBM are of the form \refEq{eq:sigmoid_p} -- \ie for a binary RBM with a sigmoid non-linearity, since $\{\FF(\eta_j)\}_j$ and $\{\FF(\nu_i)\}_i$ do not appear in the joint form \refEq{eq:joint} and $\FF^*(0) = \FF^*(1) = 0$, the joint form has the simple and the familiar form 
  $
  \pp(\xx, \yy) \, = \, \exp \big ( \xx^T \cdot \WW \cdot \yy  - \ZZ(\WW) \big )
  $.
\end{example}
\end{mdframed}
\begin{figure*}
  \centering
  \captionsetup{position=top}
\hbox{
    \subcaptionbox{\small ArcSinh unit }{\includegraphics[width=.24\textwidth]{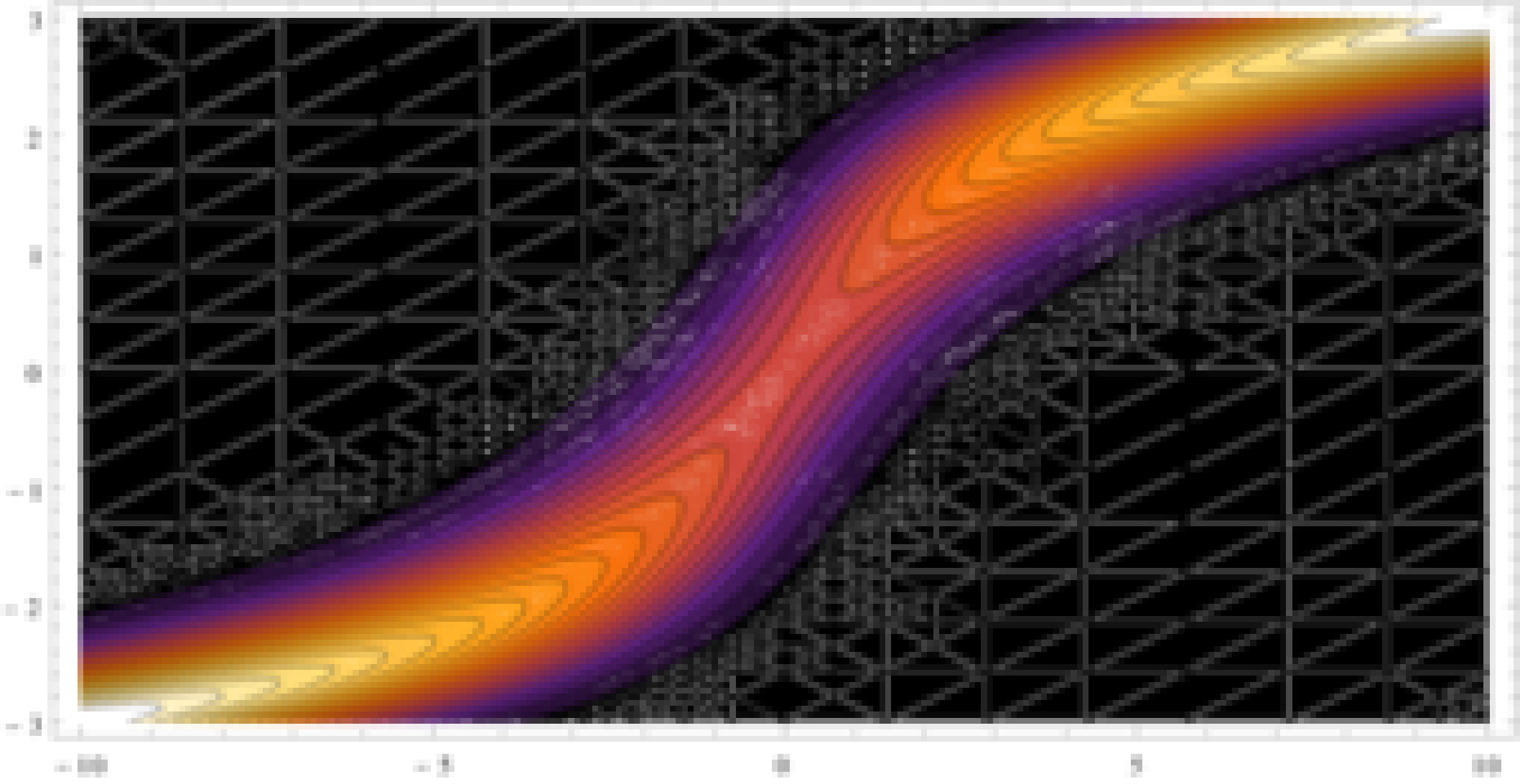}}
    \subcaptionbox{\small Sinh unit}{\includegraphics[width=.24\textwidth]{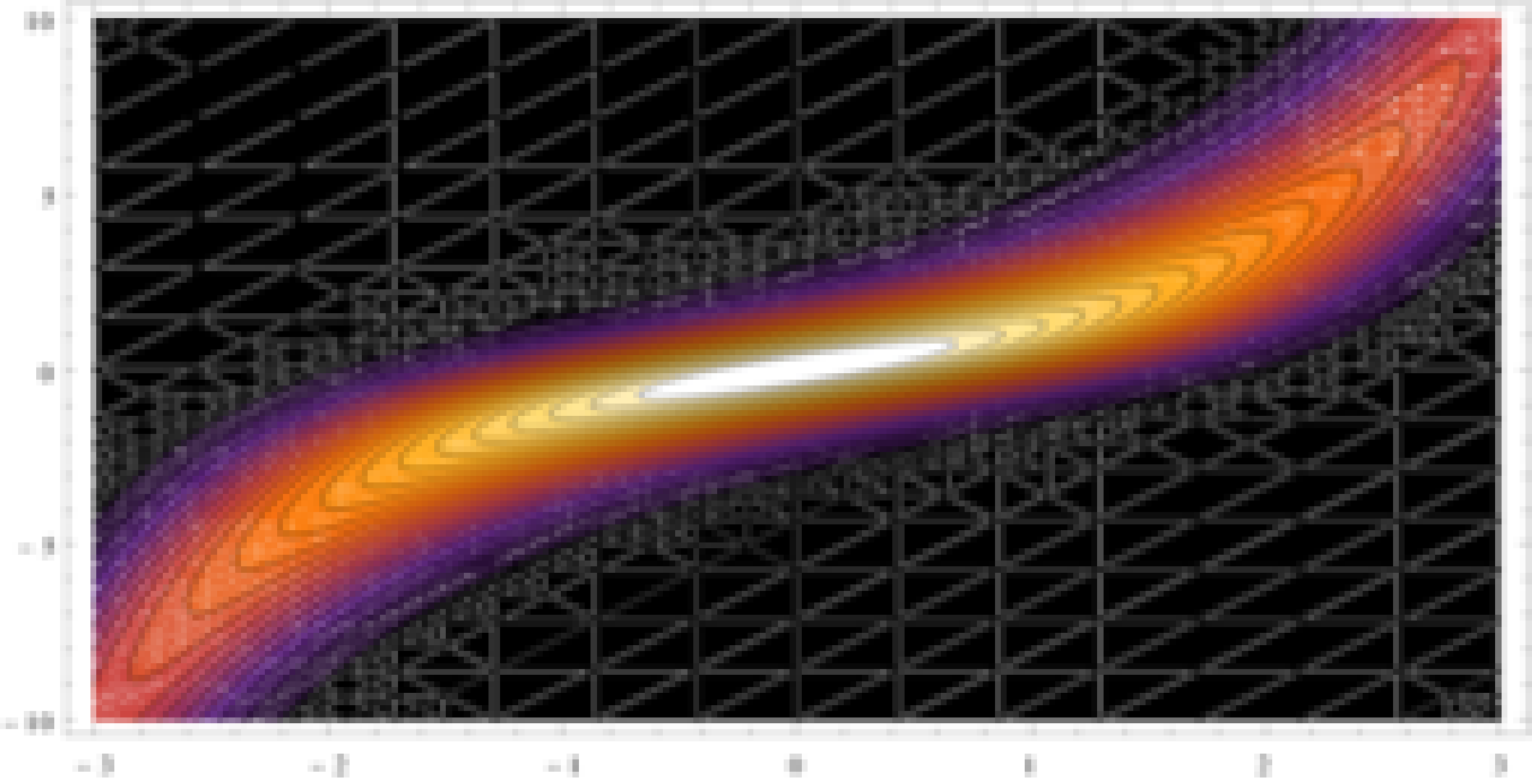} \label{fig:linear}}
    \subcaptionbox{\small Softplus unit}{\includegraphics[width=.24\textwidth]{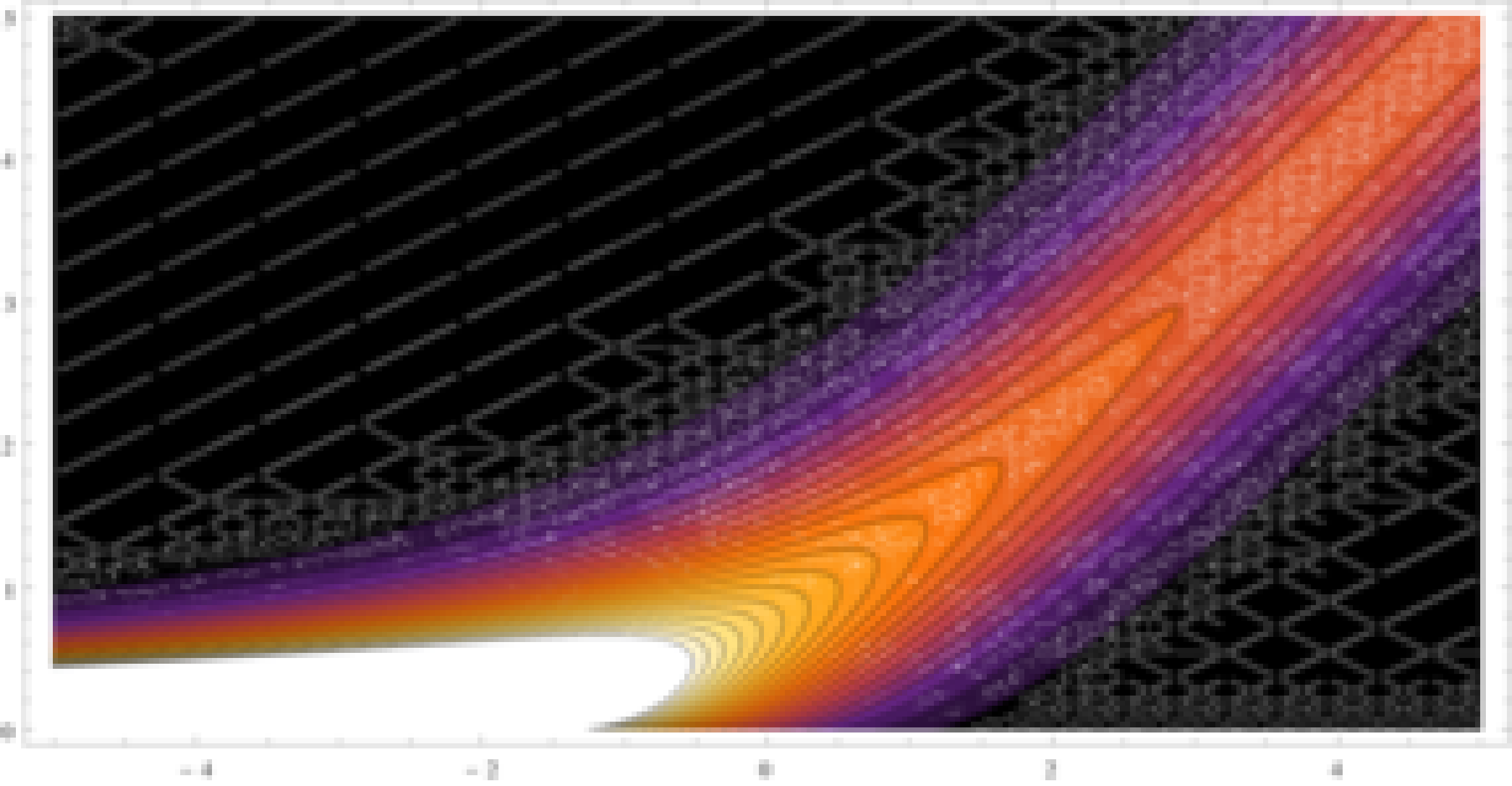}}
    \subcaptionbox{\small Exp unit}{\includegraphics[width=.24\textwidth]{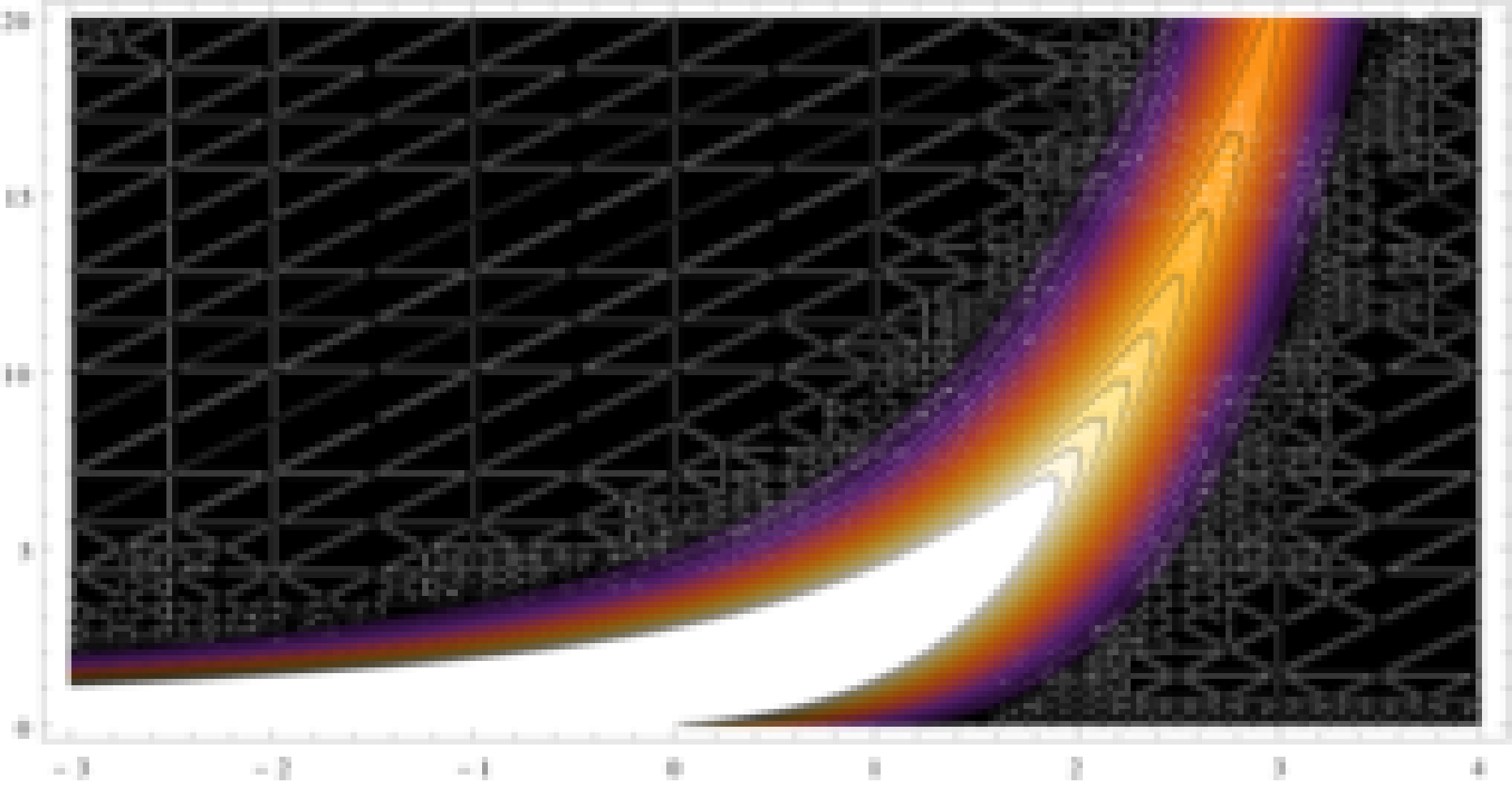}}
  }
  \hbox{
    {\includegraphics[width=.24\textwidth]{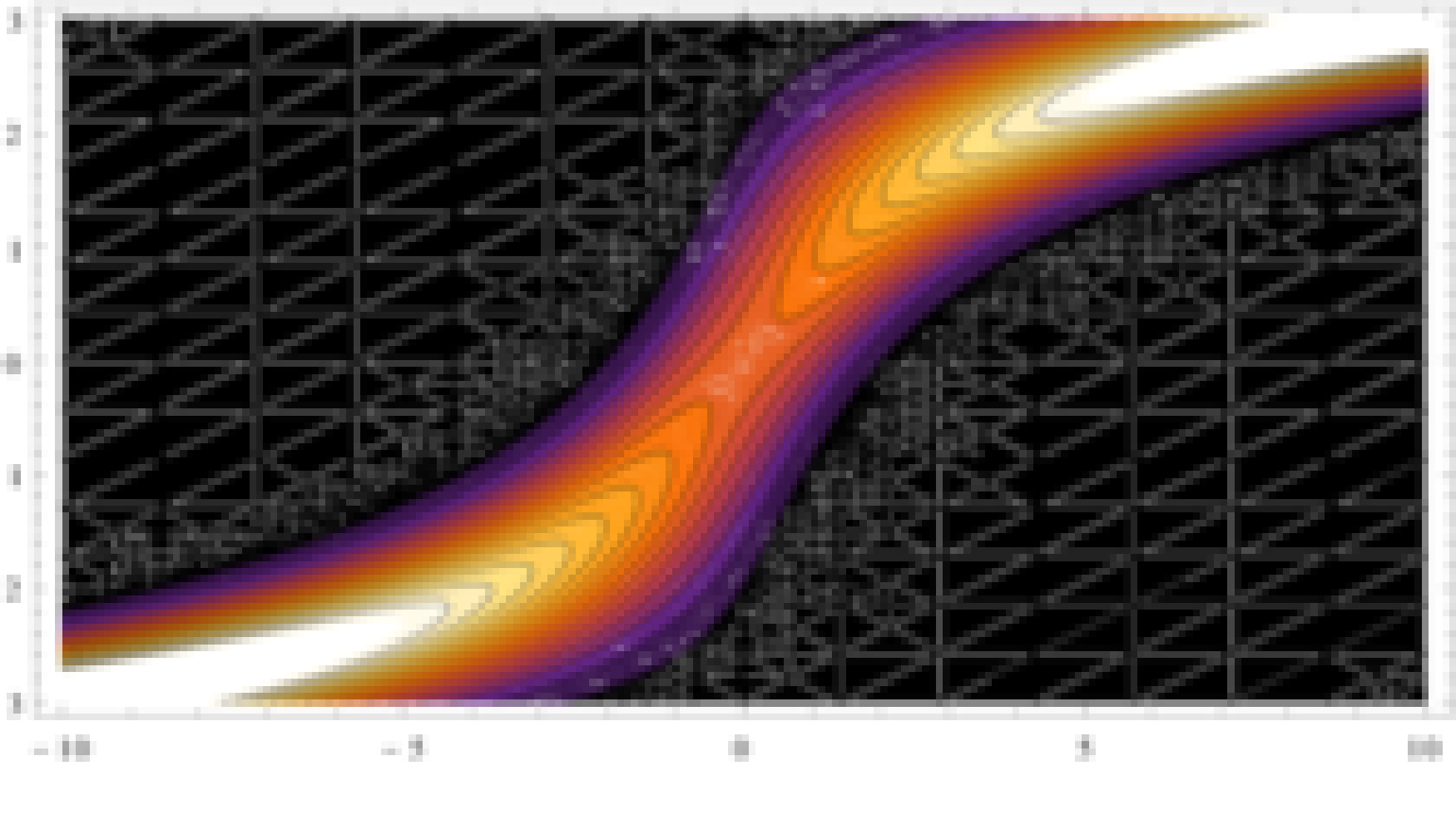}}
    {\includegraphics[width=.24\textwidth]{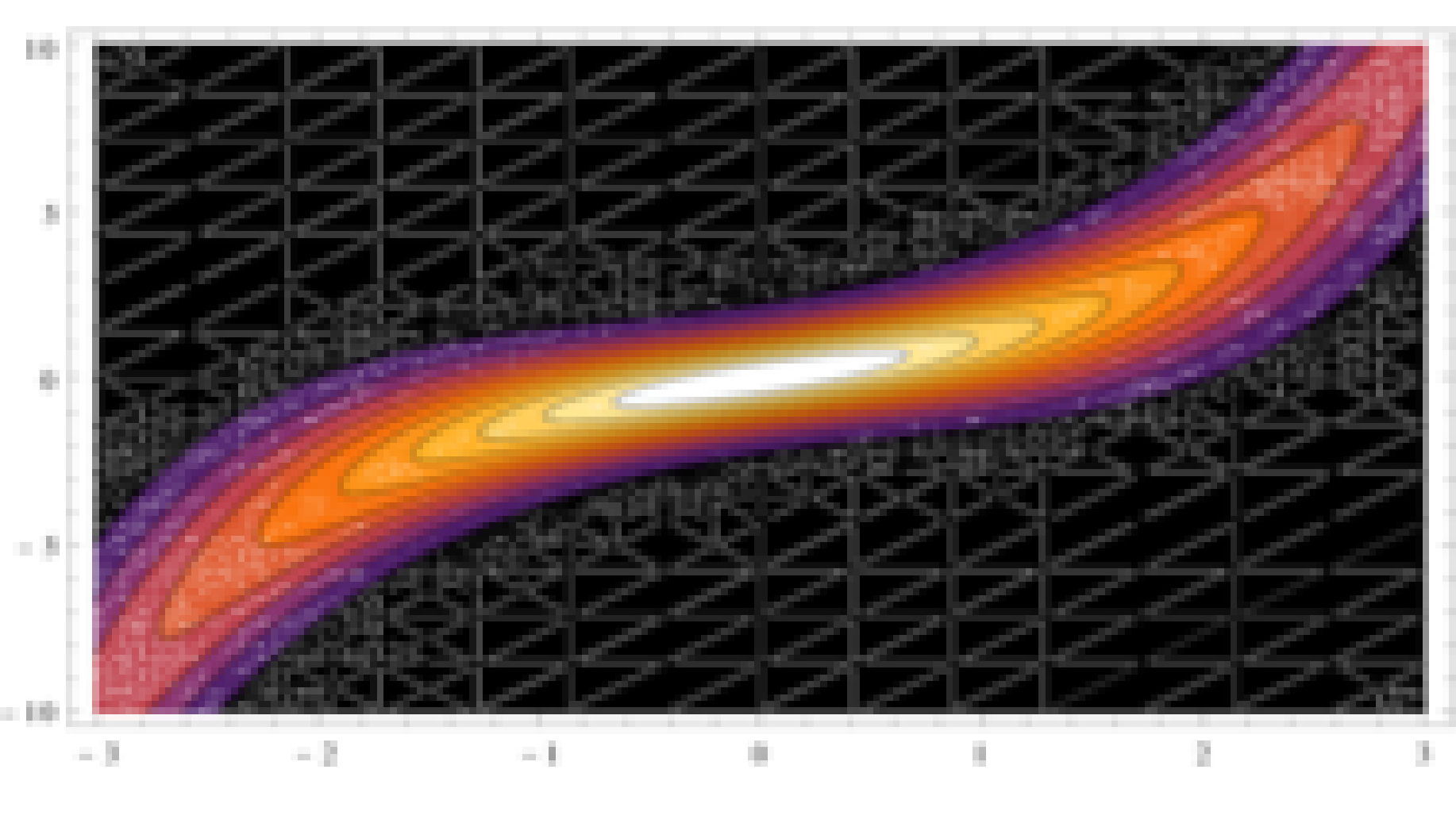}}
    {\includegraphics[width=.24\textwidth]{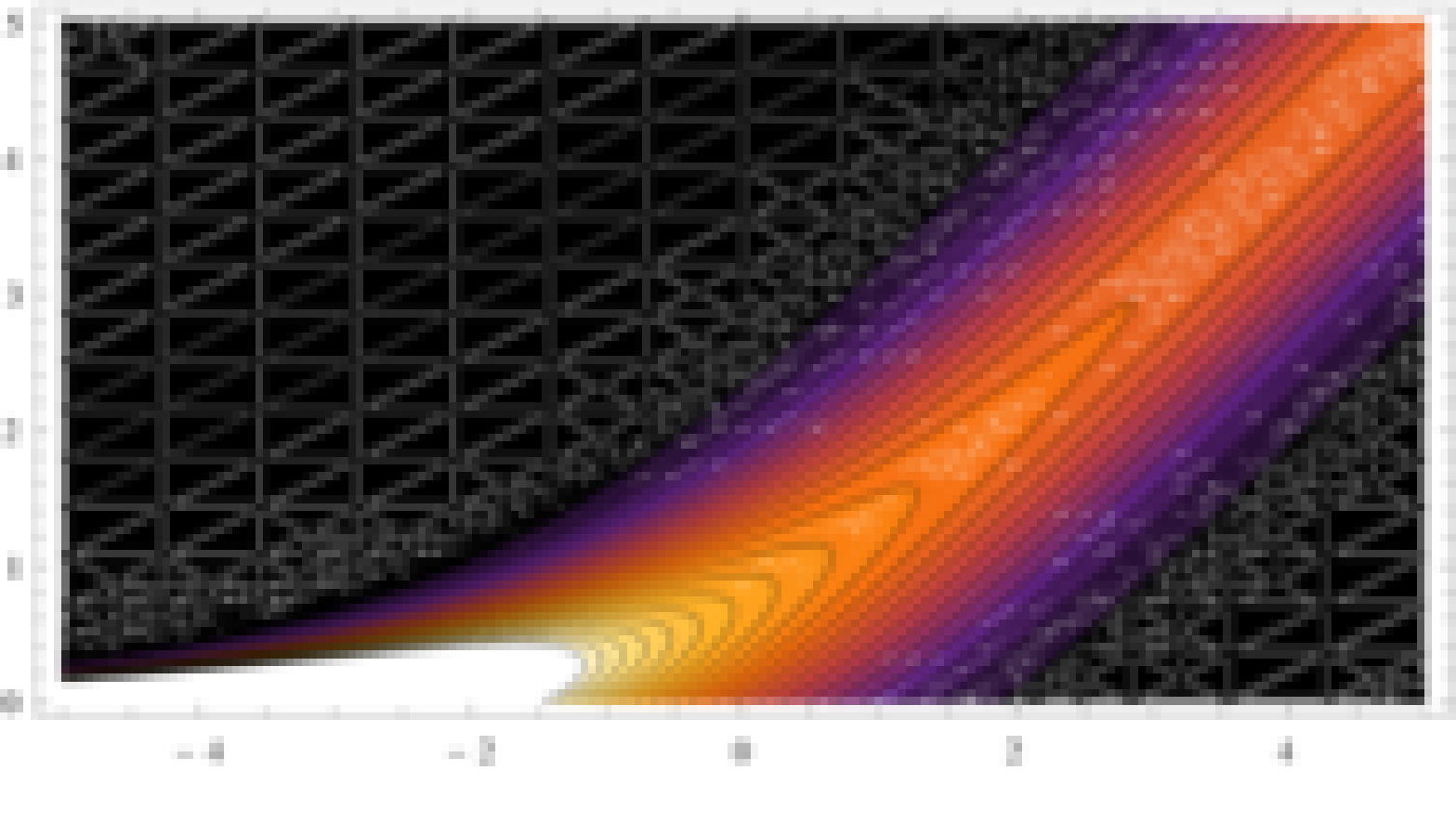}}
    {\includegraphics[width=.24\textwidth]{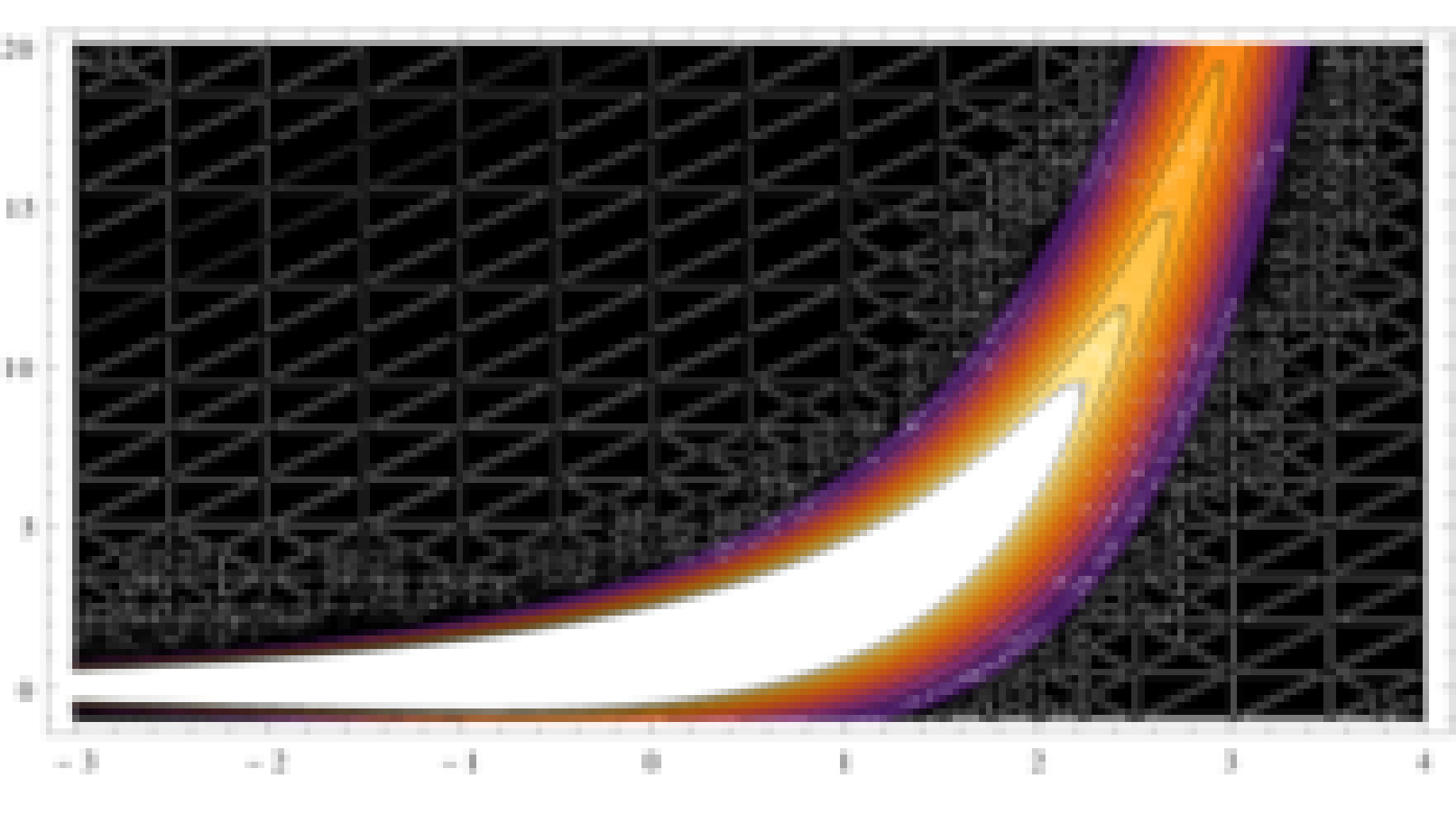}}
  }
  \caption{\it \small Conditional probability of \refEq{eq:approx1} for different 
    stochastic units (top row) and the Gaussian approximation of Proposition~\ref{prop} (bottom row) for the same unit. 
 Here the horizontal axis is the input $\eta _j= \sum_{i} W_{i,j} \x_i$ and the vertical axis is the stochastic activation $\y_j$ with the intensity $\pp(\y_j \mid \eta_j)$.
 see \Cref{table:units} for more details on these stochastic units.}
  \label{fig:comparison_sampling}
\end{figure*}

\section{Learning}\label{sec:learning}
A consistent estimator for the parameters $\WW$, given observations $\mathcal{D} = \{\xx\nn{1},\ldots,\xx\nn{N}\}$, is obtained by maximizing the marginal likelihood $\prod_{n}\pp(\xx\nn{n} \mid \WW)$, where the \refEq{eq:joint} defines the joint probability $\pp(\xx, \yy)$. 
The gradient of the log-marginal-likelihood $\nabla_{\WW} \big (\sum_{n} \log(\pp(\xx\nn{n} \mid \WW)) \big )$ is 
\begin{align}\label{eq:grad}
\frac{1}{N}\sum_{n} \EE_{\pp(\yy \mid \xx\nn{n}, \WW)} [\yy  \cdot (\xx\nn{n})^T] \; - \; \EE_{\pp(\yy, \xx \mid \WW)} [ \yy \cdot  \xx^T] \hspace{-.1in}
\end{align}
where the first expectation is \wrt the observed data 
in which $\pp(\yy \mid \xx) = \prod_j\pp(\y_j \mid \xx)$ and $\pp(\y_j \mid \xx)$ is given by \refEq{eq:bregman_p}. The second expectation is \wrt the model of \refEq{eq:joint}. 

When discriminatively training a neuron $\ff(\sum_{i} W_{i,j} \x_i)$ using input output pairs $\mathcal{D}=\{(\xx\nn{n}, \y_j\nn{n})\}_n$, 
in order to have a loss that is convex in the model parameters $\WW_{:j}$, it is common
to use a \textit{matching loss} for the given transfer function $\ff$ \citep{helmbold1999relative}. This is simply the Bregman divergence $\DD_\ff(\ff(\eta_j\nn{n}) \|\, \y_j\nn{n})$, 
where $\eta_j\nn{n} = \sum_{i} W_{i,j} \x_i\nn{n}$.
Minimizing this matching loss corresponds to maximizing the log-likelihood of \refEq{eq:bregman_p},
and it should not be surprising that the gradient $\nabla_{\WW_{:j}} \big (\sum_{n} \DD_{\ff}(\ff(\eta_j\nn{n}) \|\, \y_j\nn{n}) \big )$ of this loss \wrt $\WW_{:j} = [W_{1,j},\ldots,W_{M,j}]$ 
\begin{align*}  
\sum_{n} \ff(\eta_j\nn{n}) (\xx\nn{n})^T\;  -\; \y_j\nn{n}(\xx\nn{n})^T
\end{align*}
resembles that of \refEq{eq:grad}, where $\ff(\eta_j\nn{n})$ above substitutes $\y_j$ in \refEq{eq:grad}.

However, note that in generative training, $\y_j$ is not simply equal to $\ff(\eta_j)$, but it is sampled from the exponential family distribution \refEq{eq:bregman_p} with the mean $\ff(\eta_j)$ -- that is $\y_j = \ff(\eta_j) + \text{noise}$.
This 
extends the previous observations linking the discriminative and generative (or regularized) training -- via Gaussian noise injection -- to the noise from other members of the exponential family \citep[\eg][]{an1996effects,vincent2008extracting,bishop1995training} which in turn relates to the regularizing role of generative pretraining of neural networks~\citep{erhan2010does}. 

Our sampling scheme (next section) further suggests that
when using output Gaussian noise injection for regularization of arbitrary activation functions, the \textbf{variance of this noise should be scaled by the derivative of the activation function.}

\subsection{Sampling}\label{sec:sampling}
To learn the generative model, we need to be able to sample from the distributions that define the expectations in \refEq{eq:grad}.
Sampling from the joint model can also be reduced to alternating conditional sampling of visible and hidden variables (\ie block Gibbs sampling). Many methods, including contrastive divergence~\citep[CD;][]{hinton2002training}, stochastic maximum likelihood \citep[\aka persistent CD][]{tieleman2008training} and their variations \citep[\eg][]{tieleman2009using,breuleux2011quickly} only require this alternating sampling
in order to optimize an approximation to the gradient of \refEq{eq:grad}. 

Here, we are interested in sampling from  $\pp(\y_j \mid \eta_j)$ and $\pp(\x_i \mid \nu_i)$ as defined in \refEq{eq:bregman_p}, which is in general non-trivial. 
 However some members of the exponential family have relatively efficient sampling procedures~\citep{ahrens1974computer}. 
One of these members that we use in our experiments is the Poisson distribution.
\begin{mdframed}[style=MyFrame]
\begin{example} \label{ex:poisson}
  For a Poisson unit, a Poisson distribution 
  \begin{align}
    \label{eq:poisson_dist}
    \pp(\y_j \mid \lambda) = \frac{\lambda^{\y_j}}{\y_j!} e^{-\lambda}  
  \end{align}
  represents the probability of a neuron firing $\y_j$ times in a unit of time, given its average rate 
is $\lambda$. 
  We can define Poisson units within Exp-RBM using $\ff_j(\eta_j) = e^{\eta_j}$, which gives $\FF(\eta_j) = e^{\eta_j}$ and 
  $\FF^*(\y_j) = \y_j (\log(\y_j) - 1)$. For $\pp(\y_j \mid \eta_j)$ to be properly normalized, since $\y_j \in \mathbb{Z}^+$ is a non-negative integer,   
  $\FF^*(\y_j) + \gh(\y_j) = \log(\y_j!) \approx \FF^*(\y_j)$ (using Sterling's approximation). This gives 
    $\pp(\y_j \mid \eta_j) \; = \; \exp \big ( \y_j\eta_j - e^{\eta_j} -  \log(\y_j!) \big)$
  which is identical to distribution of \refEq{eq:poisson_dist}, for $\lambda = e^{\eta_j}$. 
This means, we can use any available sampling routine for Poisson
  distribution to learn the parameters for an exponential family RBM where some units are Poisson.
In \Cref{sec:experiments}, we use a modified version of Knuth's method~\citep{knuth1969seminumerical} for Poisson sampling.
\end{example}
\end{mdframed}

By making a simplifying assumption, the following Laplace approximation demonstrates how to use Gaussian noise to sample from general conditionals in Exp-RBM, for ``any'' smooth and monotonic non-linearity. 
\begin{proposition}\label{prop}
  Assuming a constant base measure $\gh(\y_i) = c$, the distribution of $\pp(\y_j \,\|\, \eta_j )$ is to the second order
approximated by a Gaussian 
  \begin{align}
    \label{eq:gaussian_sampling}
     \exp \bigg( -\DD_{\ff}(\eta_j\,\|\, \y_j) - c \bigg) \approx \normal(\, \y_j \mid \ff(\eta_j),\, \ff'(\eta_j)\,)
  \end{align}
  where $\ff'(\eta_j) = \frac{\dd}{\dd \eta_j} \ff(\eta_j)$ is the derivative of the activation 
  function.
\end{proposition}
\begin{proof}
  The mode (and the mean) of the conditional \refEq{eq:bregman_p}
  for $\eta_j$ is $\ff(\eta_j)$. This is because the Bregman divergence $\DD_\ff(\eta_j \| \y_j)$ achieves minimum when $\y_j = \ff(\eta_j)$.
  Now, write the Taylor series approximation to the target log-probability around its mode
  \begin{subequations}
    \begin{empheq}{align}
      &\log (\,\pp(\eps + \ff(\eta_j) \mid \eta_j \,)) \notag \\
&\;=\; \log(-\DD_\ff(\eta_j \| \eps + \ff(\eta_j))) - c\notag \\
      &\;=\; \eta_j\ff(\eta_j) - \FF^*(\ff(\eta_j)) - \FF(\eta_j) \notag \\ 
      &+ \eps (\eta_j - \ffinv(\ff(\eta_j))  + \frac{1}{2} \eps^2( \frac{-1}{\ff'(\eta_j)}) + \OO(\eps^3) \label{eq:proof1}\\
      &= \eta_j\ff(\eta_j) - (\eta_j\ff(\eta_j) - \FF(\eta_j)) - \FF(\eta_j) \label{eq:proof2}\\
      &+\eps (\eta_j - \eta_j) + \frac{1}{2} \eps^2( \frac{-1}{\ff'(\eta_j)}) + \OO(\eps^3) \notag \\
      &= -\frac{1}{2} \frac{\eps^2}{\ff'(\eta_j)}+ \OO(\eps^3) \label{eq:proof3}
    \end{empheq}
  \end{subequations}
  In \refEq{eq:proof1} we used the fact that $\frac{\dd}{\dd y}\ffinv(y) = \frac{1}{\ff'(\ffinv(y))}$
  and in \refEq{eq:proof2}, we used the conjugate duality of $\FF$ and $\FF^*$.
  Note that the final unnormalized log-probability in \refEq{eq:proof3} is that of a Gaussian, with mean zero and variance $\ff'(\eta_j)$. Since our Taylor expansion was around $\ff(\eta_j)$, this
  gives us the approximation of \refEq{eq:gaussian_sampling}.
\end{proof}


\subsubsection{Sampling Accuracy}
To exactly evaluate the accuracy of our sampling scheme, we need to 
evaluate the conditional distribution of \refEq{eq:bregman_p}.
However, we are not aware of any analytical or numeric method to estimate 
the base measure $\gh(\y_j)$. Here, we replace $\gh(\y_j)$ with $\tilde{\gh}(\eta_j)$,
playing the role of a normalization constant. We then evaluate   
\begin{align}\label{eq:approx1}
\pp( \y_j \mid \eta_j) \approx \exp \big( -\DD_{\ff}(\eta_j\,\|\, \y_j) - \tilde{\gh}(\eta_j) \big)
\end{align}
where 
$\tilde{\gh}(\eta_j)$ is numerically approximated for each $\eta_j$ value. \Cref{fig:comparison_sampling} compares this density  against the Gaussian approximation  
$\pp( \y_j \mid \eta_j) \approx \normal(\, \ff(\eta_j),\, \ff'(\eta_j)\,)$.
As the figure shows, the densities are very similar.

\subsection{Bernoulli Ensemble Interpretation}
This section gives an interpretation of Exp-RBM in terms of a Bernoulli RBM with an infinite collection of Bernoulli units. 
\citet{nair2010rectified} introduce the softplus unit, $\ff(\eta_j) = \log(1 + e^{\eta_j})$, as an approximation to the rectified linear unit (ReLU) $\ff(\eta_j) = \max(0, \eta_j)$. To have a probabilistic interpretation for this non-linearity, the authors represent it as an infinite series of
Bernoulli units with shifted bias:
\begin{align}
  \label{eq:stepped_sigmoid}
  \log(1 + e^{\eta_j}) = \sum_{n = 1}^{\infty} \sigma(\eta_j - n + .5)
\end{align}
where $\sigma(x) = \frac{1}{1 + e^{-x}}$ is the sigmoid function. This means that the sample $y_j$ from a softplus unit is effectively the number of active Bernoulli units. 
The authors then suggest using $\y_j \sim
\max(0, \normal(\eta_j, \sigma(\eta_j))$ to sample from this type of unit. In comparison, our Proposition~\ref{prop} suggests using  $\y_j \sim \normal(\log(1 + e^{\eta_j}), \sigma(\eta_j))$ for softplus and $\y_j \sim \normal(\max(0, \eta_j), step(\eta_j))$ -- where $step(\eta_j)$ is the step function -- for ReLU. Both of these are very similar to the approximation of \citep{nair2010rectified} and we found them to perform similarly in practice as well. 

Note that these Gaussian approximations are assuming $\gh(\eta_j)$ is constant. 
However, by numerically approximating $\int_{\YY_j} \exp\big(-\DD_{\ff}(\eta_j \| \y_j)\big) \dd \y_j$, for $\ff(\eta_j) = \log(1 + e^{\eta_j})$, 
\Cref{fig:softplus_error} shows that the  integrals are not the same for different values of $\eta_j$, showing that the base measure $\gh(\y_j)$ is not constant for ReLU. 
In spite of this, experimental results for pretraining ReLU units using Gaussian noise
suggests the usefulness of this type of approximation. 

\begin{figure}
  \centering
\includegraphics[width=.4\textwidth]{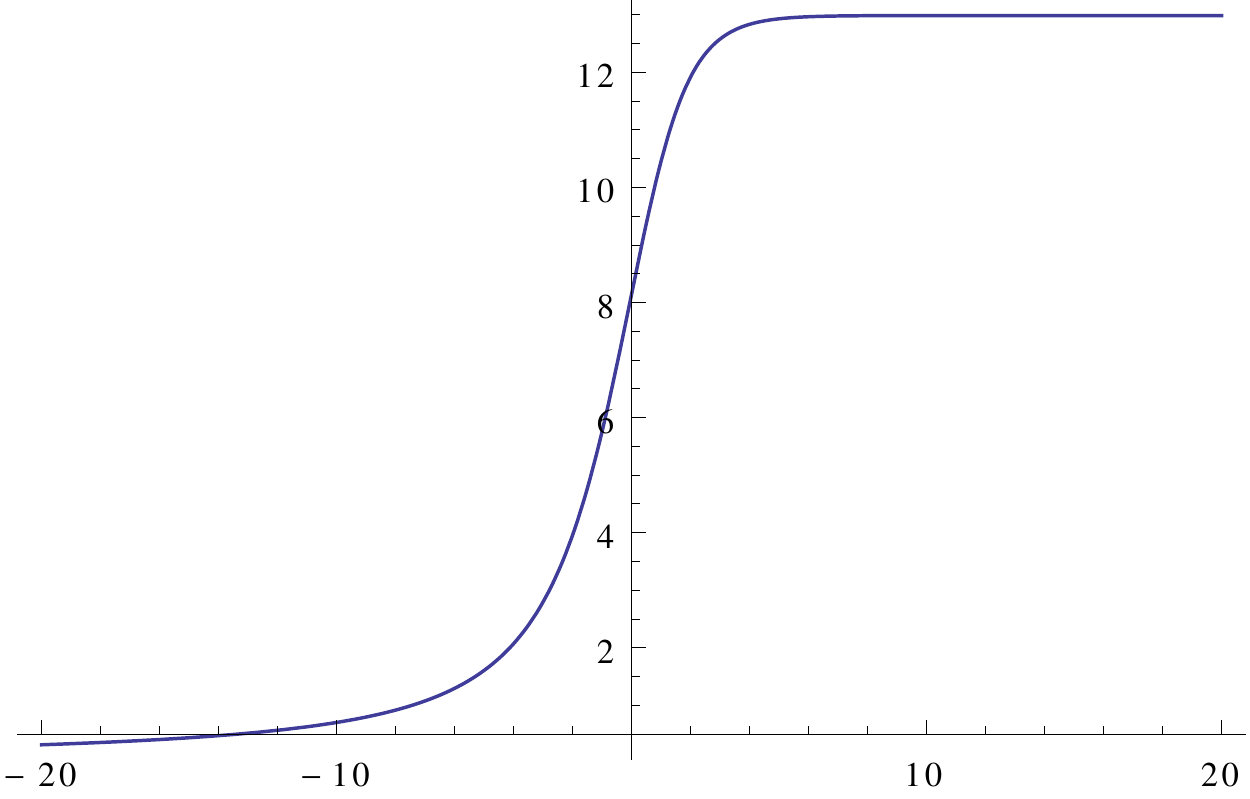} 
  \caption{\it \small  Numerical approximation to the integral $\int_{\YY_j} \exp\big(-\DD_{\ff}(\eta_j \| \y_j)\big) \dd \y_j$ for the softplus unit $\ff(\eta_j) = \log(1 + e^{\eta_j})$, at different $\eta_j$.} \label{fig:softplus_error}
\end{figure}

We can extend this interpretation as a collection of (weighted) Bernoulli units to any non-linearity $\ff$. For simplicity, let us assume $\lim_{\eta \to -\infty} \ff(\eta) = 0$ and $\lim_{\eta \to +\infty} \ff(\eta) = \infty$\footnote{The following series and the sigmoid function need to be adjusted depending on these limits. For example, for the case where $\y_j$ is antisymmetric and unbounded (\eg $\ff(\eta_j) \in \{\,\mathrm{sinh}(\eta_j),\, \mathrm{sinh^{-1}}(\eta_j),\, \eta_j\, | \eta_j|\} $), we need to change the domain of Bernoulli units from $\{0,1\}$ to $\{-.5,+.5\}$. This corresponds to changing the sigmoid to hyperbolic tangent $\frac{1}{2}\, \mathrm{tanh}(\frac{1}{2} \eta_j)$. In this case, we also need to change the bounds for $n$ in the series of \refEq{eq:series} to $\pm \infty$.}, and define the following series of Bernoulli units:
$
\sum_{n=0}^{\infty} \alpha \sigma(\ffinv(\alpha n))
$,
where the given parameter
$\alpha$ 
is the weight of each unit. Here, we are defining a new Bernoulli unit with a weight $\alpha$ for each $\alpha$ unit of change in the value of $\ff$. Note that the underlying idea is similar to that of inverse transform sampling~\citep{devroye1986non}.
At the limit of $\alpha \to 0^+$ we have
\begin{align}\label{eq:series}
  \ff(\eta_j) \approx \alpha \sum_{n=0}^{\infty}  \sigma(\eta_j - \ffinv(\alpha n))
\end{align}
that is $\hat{\y}_j \sim \pp(\y_j \mid \eta_j)$ is the weighted sum of active Bernoulli units. 
\Cref{fig:activations}(a) shows the approximation of this series for the softplus function for decreasing values of $\alpha$.
\begin{figure}
  \centering
\includegraphics[width=.6\textwidth]{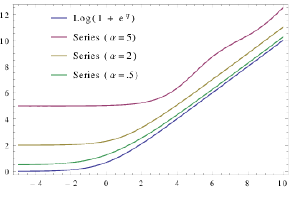} \label{fig:series}
  \caption{\it \small  reconstruction of ReLU by as a series of Bernoulli units with shifted bias.}
\end{figure}


\section{Experiments and Discussion}\label{sec:experiments}
We evaluate the representation capabilities of Exp-RBM for different stochastic units in the following two sections.
Our initial attempt was to adapt Annealed Importance Sampling \citep[AIS;][]{salakhutdinov2008quantitative} to Exp-RBMs. 
However, estimation of the importance sampling ratio in AIS for general Exp-RBM proved challenging. We consider two alternatives: 1) for large datasets, \refSection{sec:filters} qualitatively evaluates the filters learned by various units and; 2) \refSection{sec:quant} evaluates Exp-RBMs on a smaller
dataset where we can use indirect sampling likelihood to quantify the generative quality of the models with different activation functions.
 
Our objective here is to demonstrate that a combination of our sampling scheme with contrastive divergence (CD) training can indeed
produce generative models for a diverse choice of activation function. 

\subsection{Learning Filters}\label{sec:filters}
\begin{figure}
  \centering
  \includegraphics[width=1.\textwidth]{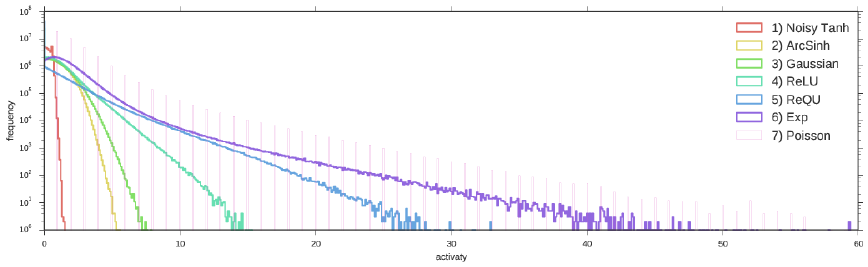}
  \caption{\it \small Histogram of hidden variable activities on the MNIST test data, for different types of units. Units with heavier tails produce longer 
    strokes in \Cref{fig:mnist}. Note that the linear decay of activities in the log-domain correspond to exponential decay with different exponential coefficients.}\label{fig:activations}
\end{figure}

\begin{figure}
  \begin{tabular}{r@{\hskip -0pt}c@{\hskip 0 pt}l}
    \scalebox{.6}{\rotatebox{90}{data}} & \includegraphics[width=.95\textwidth]{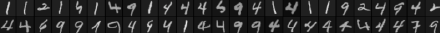} & \\[-.01in] 
    \scalebox{.5}{\rotatebox{90}{N.Tanh}} & \includegraphics[width=.95\textwidth]{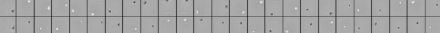}
                                    & \scalebox{.5}{\rotatebox{90}{bounded}}\\[-.02in] \cline{3-3}
    \scalebox{.5}{\rotatebox{90}{arcSinh}} & \includegraphics[width=.95\textwidth]{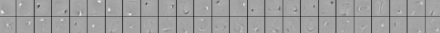}
                                                                                                                  & \tiny \rotatebox{90}{\ log.} \\[-.01in]\cline{3-3}
    \scalebox{.6}{\rotatebox{90}{SymSq}} & 
                                   \includegraphics[width=.95\textwidth]{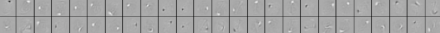}
                                                                                                                  & \tiny \rotatebox{90}{\ \ \ sqrt} \\[-.01in]\cline{3-3}
    \scalebox{.6}{\rotatebox{90}{ReL}} & \includegraphics[width=.95\textwidth]{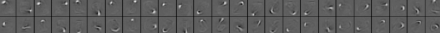}
                                                                                                                  & \tiny \rotatebox{90}{linear} \\[-.01in] \cline{3-3}
    \scalebox{.6}{\rotatebox{90}{ReQ}} & \includegraphics[width=.95\textwidth]{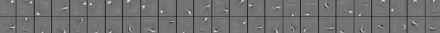}
                                                                                                                  & \multirow{2}{*}{\tiny \rotatebox{90}{ quadratic}}\\[-.01in] 
    \scalebox{.6}{\rotatebox{90}{SymQ}} & \includegraphics[width=.95\textwidth]{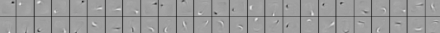}\\[-.02in] \cline{3-3}
    \scalebox{.6}{\rotatebox{90}{Sinh}} & \includegraphics[width=.95\textwidth]{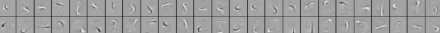}
                                                                                                                  & \multirow{3}{*}{\tiny \rotatebox{90}{ exponential}}\\[-.01in] 
    \scalebox{.6}{\rotatebox{90}{\hspace{.1in}Exp}} & \includegraphics[width=.95\textwidth]{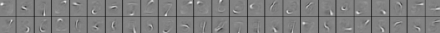}\\[-.00in] 
    \scalebox{.6}{\rotatebox{90}{Poisson}} & \includegraphics[width=.95\textwidth]{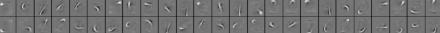}\\[-.01in] \cline{3-3}
  \end{tabular}
  \caption{\small \it Samples from the MNIST dataset (first two rows) and the filters with highest variance for different Exp-RBM stochastic units (two rows per unit type). 
    From top to bottom the non-linearities grow more rapidly, also producing features that represent longer strokes.}
  \label{fig:mnist}
\end{figure}

In this section,
we used CD
with a single Gibbs sampling step, $1000$ hidden units, Gaussian visible units\footnote{Using Gaussian visible units also assumes that the input data is normalized to have a standard deviation of 1.}, mini-batches and method of momentum, and selected the 
learning rate from $\{10^{-2}, 10^{-3}, 10^{-4}\}$ using reconstruction error at the final epoch.

The MNIST handwritten digits dataset~\citep{lecun1998gradient} 
is a dataset of 70,000 ``size-normalized and centered'' binary images.
Each image is $28 \times 28$ pixel, and represents one of $\{0,1,\ldots,9\}$ digits. See the first row of \Cref{fig:mnist} for few instances from MNIST dataset.
For this dataset we use a momentum of $.9$ and train each model for 25 epochs. \Cref{fig:mnist} shows the filters of different stochastic units;
see \Cref{table:units} for details on different stochastic units.
Here, the units are ordered based on the asymptotic behavior of the activation function $\ff$; see the right margin of the figure.
This asymptotic change in the activation function is also evident from the hidden unit activation histogram of \Cref{fig:activations}(b), where the activation are produced on the test set
using the trained model.

These two figures suggest that transfer functions with faster asymptotic growth, have a more heavy-tailed distributions of
activations and longer strokes for the MNIST dataset, also hinting that they may be preferable
in learning representation~\citep[\eg see][]{olshausen1997sparse}. However, this comes at the cost of train-ability.
In particular, for all exponential units, due to occasionally large gradients, we have to reduce the learning rate to $10^{-4}$ while the Sigmoid/Tanh unit remains stable for a learning rate of $10^{-2}$. Other factors that affect the instability of training for exponential and quadratic Exp-RBMs are large momentum and small number of hidden units. Initialization of the weights could also play an important role, and sparse initialization
 ~\citep{sutskever2013importance,martens2010deep} and regularization schemes~\citep{goodfellow2013maxout} could potentially improve the training of these models. In all experiments, we used uniformly random values in $[-.01,.01]$ for all unit types.
In terms of training time, different Exp-RBMs that use the Gaussian noise and/or Sigmoid/Tanh units have similar computation time on both CPU and GPU.


\Cref{fig:svhn}(top) shows the receptive fields for the street-view house numbers (SVHN)~\citep{netzer2011reading} dataset. This dataset contains 600,000 images of digits in natural settings. Each image contains three RGB values for
$32 \times 32$ pixels. \Cref{fig:svhn}(bottom) shows few filters obtained from
the jittered-cluttered NORB dataset~\citep{lecun2004learning}. NORB dataset contains 291,600 stereo $2 \times (108 \times 108)$ images of 50 toys under different lighting, angle and backgrounds. Here, we use a sub-sampled $48 \times 48$ variation, and report the features learned by two types of neurons. For learning from these two datasets, we increased the momentum to $.95$ and trained different models using up to 50 epochs.

\begin{figure}
  \begin{tabular}{l@{\hskip -0pt}c@{\hskip -0pt}l}
    \scalebox{.6}{\rotatebox{90}{dataset}} & \includegraphics[width=.93\textwidth]{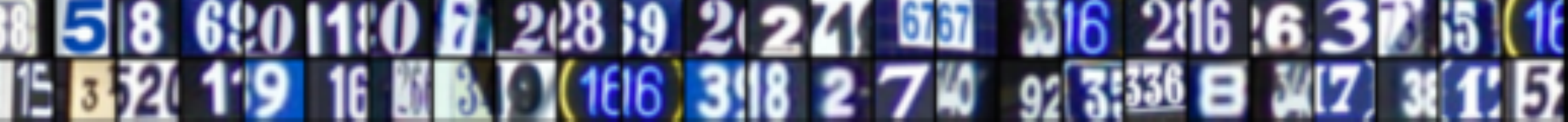} &  \multirow{3}{*}{\tiny \rotatebox{90}{ SVHN}}\\[-.02in]  
    \scalebox{.6}{\rotatebox{90}{sigmoid}} & \includegraphics[width=.93\textwidth]{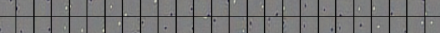} \\[-.02in] 
 \scalebox{.6}{\rotatebox{90}{ReQU}} &                                  \includegraphics[width=.93\textwidth]{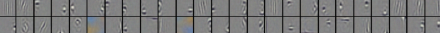} \\[0in]  
    \scalebox{.6}{\rotatebox{90}{dataset}} & \includegraphics[width=.93\textwidth]{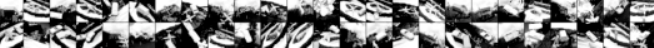} 
& \multirow{3}{*}{\tiny\rotatebox{90}{NORB}}\\[-.02in]  
    \scalebox{.6}{\rotatebox{90}{Tanh}} & 
\includegraphics[width=.93\textwidth]{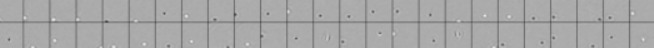} \\[-.02in]  
 \scalebox{.6}{\rotatebox{90}{SymQU}} &
\includegraphics[width=.93\textwidth]{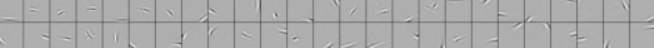} \\[-.02in] 
  \end{tabular}
  \caption{\small \it Samples  and the receptive fields of different stochastic units for 
from the \textbf{(top three rows)} SVHN dataset and \textbf{(bottom three rows)} $48 \times 48$ (non-stereo) NORB dataset with jittered objects and cluttered background. Selection of the receptive fields is based on their variance.}
  \label{fig:svhn}
\end{figure}

\subsection{Generating Samples}\label{sec:quant}
The USPS dataset~\citep{hull1994database} is relatively smaller dataset of 9,298, $16\times16$
 digits. We binarized this data and used $90\%$, $5\%$ and $5\%$ of instances for training, validation and test respectively; see \Cref{fig:samples} (first two rows) for instances from this dataset. We used Tanh activation function for the $16 \times 16 = 256$ \textit{visible} units of the Exp-RBMs%
\footnote{Tanh unit is similar to the sigmoid/Bernoulli unit, with the difference that it is (anti)symmetric $\x_i \in \{-.5, +.5\}$.}
and 500 hidden units of different types: 1) Tanh unit; 2) ReLU; 3) ReQU and 4)Sinh unit.

\begin{figure}
  \begin{tabular}{l@{\hskip -0pt}c}
    \scalebox{.6}{\rotatebox{90}{dataset}} & \includegraphics[width=.93\textwidth]{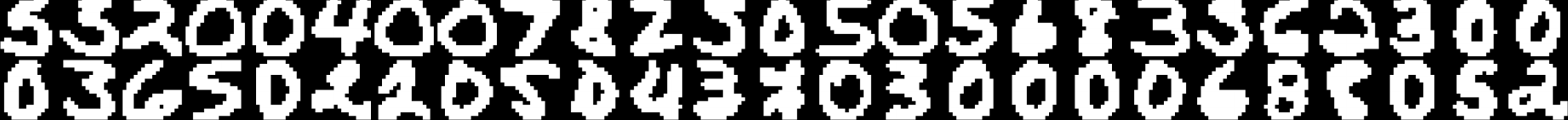} \\[-.02in] 
    \scalebox{.6}{\rotatebox{90}{Tanh}} & \includegraphics[width=.93\textwidth]{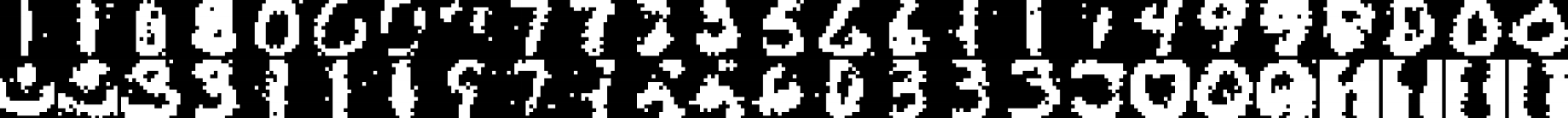} \\[-.02in] 
 \scalebox{.6}{\rotatebox{90}{ReL}} & \includegraphics[width=.93\textwidth]{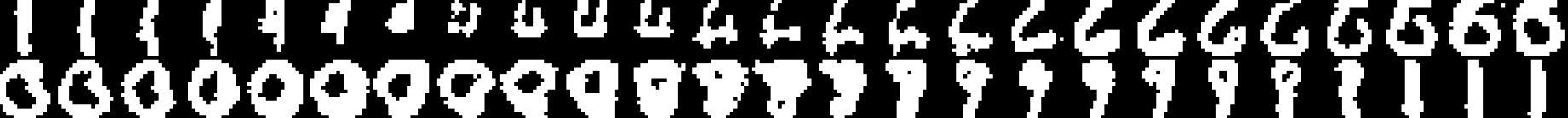} \\[-.02in]  
    \scalebox{.6}{\rotatebox{90}{ReQ}} & \includegraphics[width=.93\textwidth]{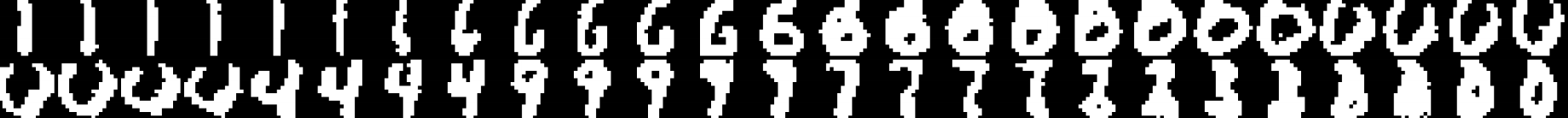}\\[-.02in]  
\scalebox{.6}{\rotatebox{90}{Sinh}} &
\includegraphics[width=.93\textwidth]{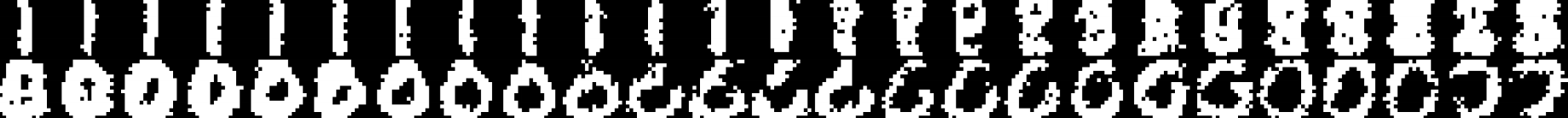}
  \end{tabular}
  \caption{\small \it Samples from the USPS dataset (first two rows) and few of the consecutive samples generated from different Exp-RBMs using rates-FPCD.}
  \label{fig:samples}
\end{figure}


We then trained these models using CD with 10 Gibbs sampling steps. 
Our choice of CD rather than 
alternatives that are known to produce better generative models, such as Persistent CD \citep[PCD;][]{tieleman2008training}, fast PCD~\citep[FPCD;][]{tieleman2009using} and  
\citep[rates-FPCD;][]{breuleux2011quickly} is due to practical reasons; these alternatives were unstable for some activation functions, while CD was always well-behaved. 
We ran CD for 10,000 epochs with three different learning rates $\{.05, .01,.001\}$ for each model. Note that here, we did not use method of momentum and mini-batches in order to to minimize the number of hyper-parameters for our quantitative comparison. We used rates-FPCD \footnote{ We used $10$ Gibbs sampling steps for each sample, zero decay of fast weights -- as suggested in ~\citep{breuleux2011quickly} -- and three different fast rates $\{.01, .001,.0001\}$.} to generate  $9298 \times \frac{90}{100}$  samples from each model -- \ie the same number as the samples in the training set. We produce these sampled datasets every 1000 epochs.
\Cref{fig:samples} shows the samples generated by different models at their final epoch, for the ``best choices'' of sampling parameters and learning rate.

\begin{figure}
  \centering
\hbox{
\includegraphics[width=.55\textwidth]{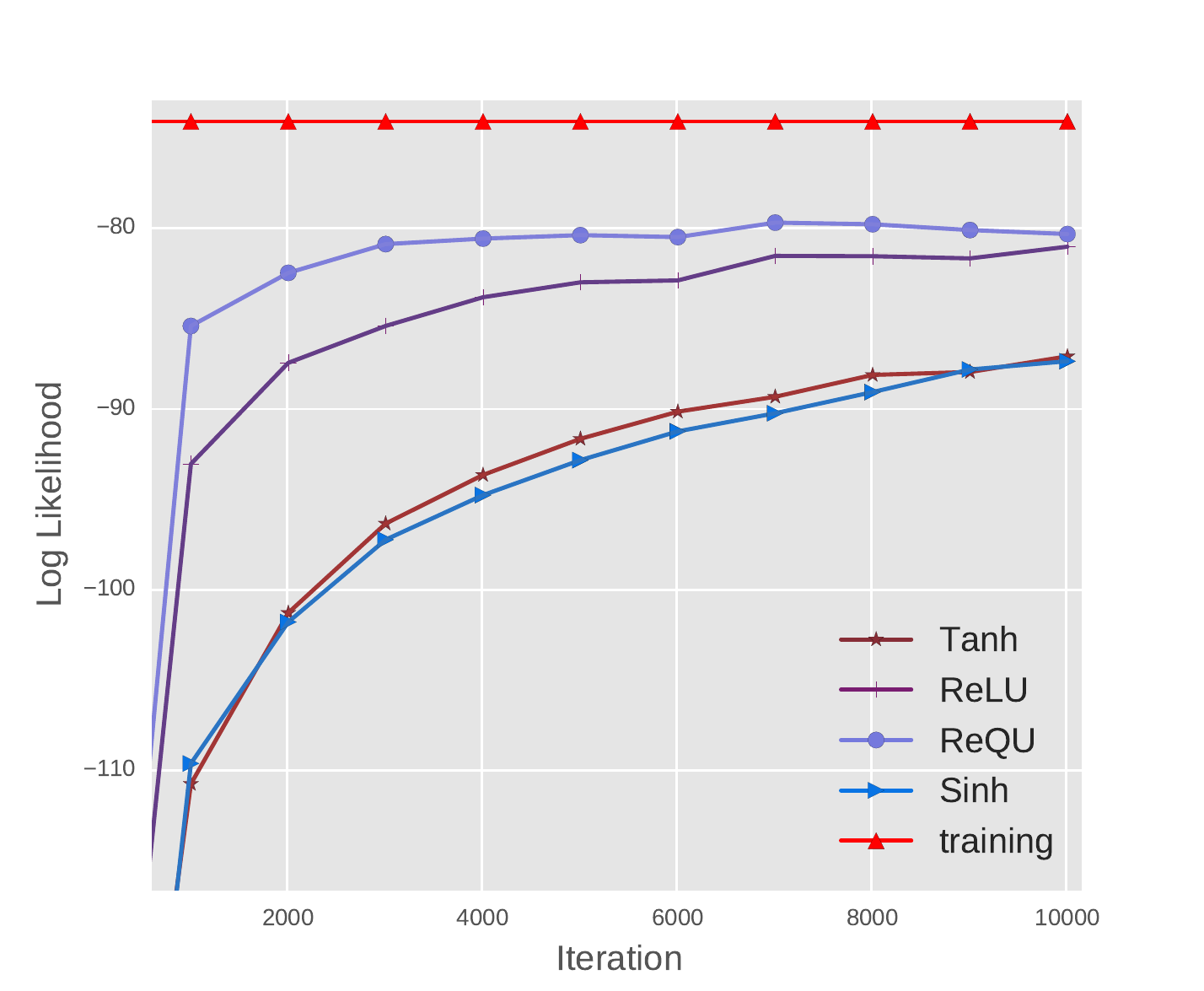} 
\includegraphics[width=.45\textwidth]{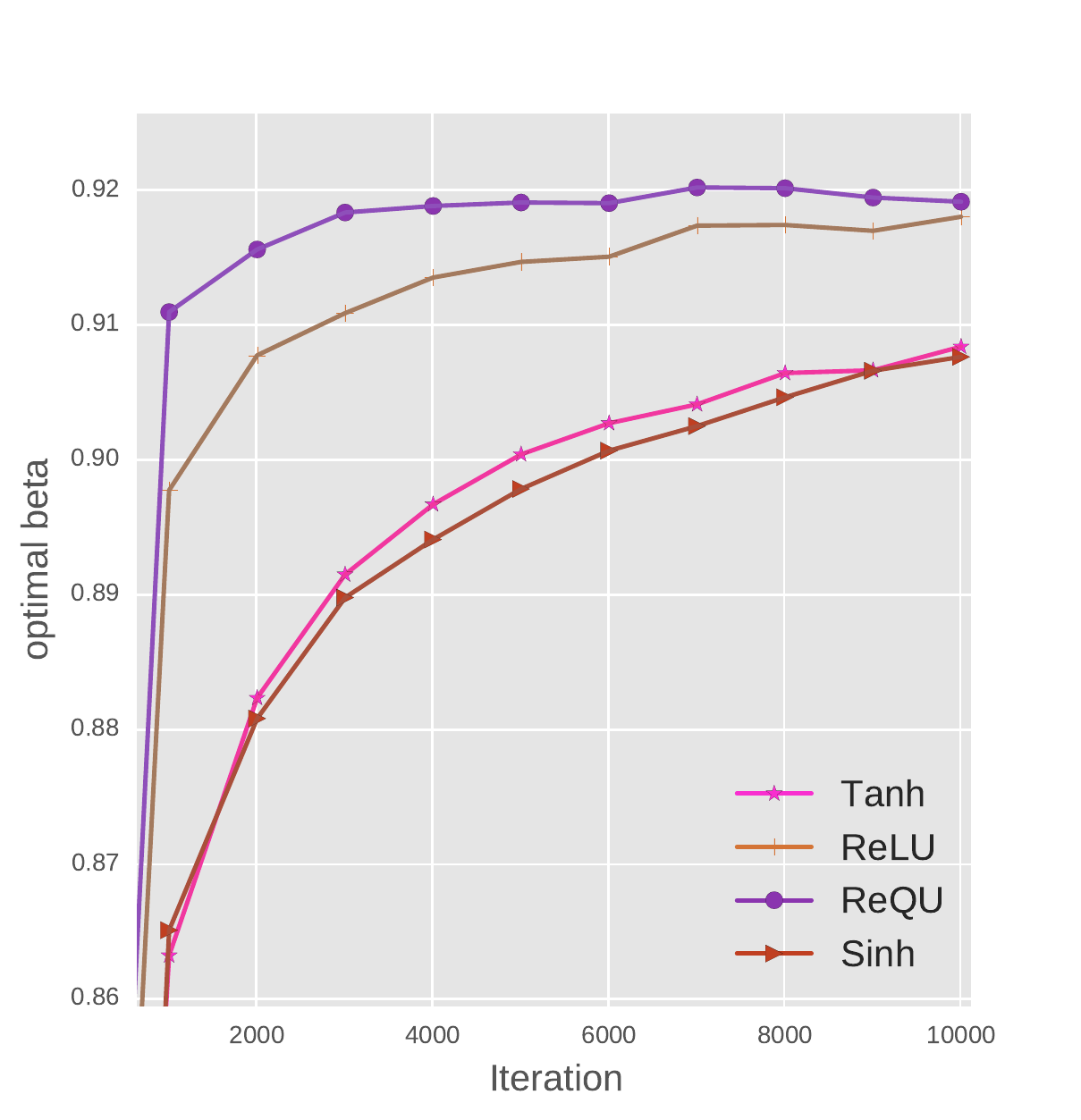} 
}
  \caption{\it \small  Indirect Sampling Likelihood of the test data (left) and $\beta^*$ for the density estimate (right) at different epochs (x-axis) for USPS dataset. }\label{fig:ISL}
\end{figure}

We then used these samples $\DD_{sample} = \{\xx\nn{1},\ldots,\xx\nn{N=9298}\}$, from each model to estimate the Indirect Sampling Likelihood \citep[ISL;][]{breuleux2011quickly} of the validation set. For this, we built a non-parametric density estimate 
\begin{align}
  \label{eq:kde}
  \hat{\pp}(\xx; \beta) = \sum_{n=1}^{N} \prod_{j=1}^{256} \beta^{\identt(\x_j\nn{n} = \x_j)} (1 - \beta)^{\identt(\x_j\nn{n} \neq \x_j)}
\end{align}
and optimized the parameter $\beta \in (.5,1)$ to maximize the likelihood of the validation set -- that is $\beta^* = \arg_{\beta}\max\, \prod_{\xx \in \DD_{valid}} \hat{\pp}(\xx, \beta)$. Here,  $\beta = .5$ defines a uniform distribution over all possible binary images, while for $\beta = 1$, only the training instances have a non-zero probability.

We then used the density estimate for $\beta^*$ as well as the best rates-FPCD sampling parameter to evaluate the ISL of the \textit{test set}.
At this point, we have an estimate of the likelihood of test data for each hidden unit type, for every 1000 iteration of CD updates. 
The likelihood of the test data using the density estimate produced \textit{directly from the training data}, gives us an upper-bound on the ISL of these models.

\Cref{fig:ISL} presents all these quantities: for each hidden unit type, we present the results for the learning rate that achieves the highest ISL. 
The figure shows the estimated log-likelihood  (left) as well as $\beta^*$ (right) as a function of the number of epochs. 
As the number of iterations increases, all models produce samples
that are more representative (and closer to the training-set likelihood). This is also consistent with $\beta^*$ values getting closer to $\beta^*_{training} = .93$, the optimal parameter for the training set.

In general, we found stochastic units defined using ReLU and Sigmoid/Tanh to be the most numerically stable. However, for this problem, ReQU learns the best model and even by increasing the CD steps to 25 and also increasing the epochs by a factor of two we could not produce
similar results using Tanh units.  
This shows that a non-linearities outside the circle of well-known and commonly used exponential family, can sometimes produce more powerful generative models, even using an ``approximate'' sampling procedure.

\section*{Conclusion}
This paper studies a subset of exponential family Harmoniums (EFH) with a single sufficient statistics for the purpose of learning generative models. 
The resulting family of distributions, Exp-RBM, gives a freedom of choice for the activation function of individual units, paralleling the freedom in discriminative training of neural networks. 
Moreover, it is possible to efficiently train arbitrary members of this family.
For this, we introduced a principled and efficient approximate sampling procedure and demonstrated 
that various Exp-RBMs can learn useful generative models and filters.

\bibliographystyle{plainnat}
  \bibliography{refs}
\clearpage
\appendix

\SetKwInput{KwInput}{Input}
\SetKwInput{KwOutput}{Output}
\SetEndCharOfAlgoLine{}
\SetKw{init}{Initialize}
\SetKw{fix}{Fix}
\begin{algorithm}[h]
\caption{Training Exp-RBMs using contrastive divergence}\label{alg:1}
\KwInput{training data $\DD = \{\x\nn{n}\}_{1 \leq n \leq N}$ ;\#CD steps; \#epochs; learning rate $\lambda$; activation functions $\{\ff(\x_i)\}_i, \{\ff(\y_j)\}_j$}
\KwOutput{model parameters $\WW$}
\init{$\WW$}\;
\For{\#epochs}{
\tcc{positive phase (+)}
$^{+}\eta_j\nn{n} = \sum_{i} \W_{i,j}\, ^{+}\xx\nn{n}_{i} \quad \forall j,n$\;
\lIf{using Gaussian apprx.}{
  $^{+}\y\nn{n}_j \sim \normal(\ff(^{+}\eta\nn{n}_j), \ff'(^{+}\eta\nn{n}_j)) \quad \forall j,n$}
\lElse{$^{+}\y\nn{n}_j \sim \pp(\y_j \mid ^+\xx\nn{n}) \quad \forall j,n$}

$^-\yy\nn{n} \leftarrow ^+\yy\nn{n} \quad \forall n$\;

\tcc{negative phase (-)}
\For{\#CD steps}{
$^{-}\nu_i\nn{n} = \sum_{j} \W_{i,j}\, ^-\y\nn{n}_{i} \quad \forall i,n$\;
\lIf{using Gaussian apprx.}{
  $^{-}\x\nn{n}_i \sim \normal(\ff(^{-}\nu\nn{n}_i), \ff'(^{-}\nu\nn{n}_i)) \quad \forall i,n$}
\lElse{$^{-}\x\nn{n}_j \sim \pp(\x_j \mid \yy\nn{n}) \quad \forall i,n$}

$^{-}\eta_j\nn{n} = \sum_{i} \W_{i,j}\, {}^{-}\xx\nn{n}_{i} \quad \forall j,n$\;
\lIf{using Gaussian apprx.}{
  $^{-}\y\nn{n}_j \sim \normal(\ff(^{-}\eta\nn{n}_j), \ff'(^{-}\eta\nn{n}_j)) \quad \forall j,n$}
\lElse{$^{+}\y\nn{n}_j \sim \pp(\y_j \mid ^-\xx\nn{n}) \quad \forall j,n$}
}
 $\W_{i,j} \leftarrow \W_{i,j} + \lambda \bigg ((^+\x_i\,^{+}\y_j) - (^-\x_i\,^{-}\y_j) \bigg ) \forall i,j$\;
}
\end{algorithm}

\end{document}